\documentclass[final,12pt]{alt2024} % Anonymized submission
\usepackage{commands_alt}

\title[Sample Complexity of Two-Layer Network]{On the Sample Complexity of Two-Layer Networks: Lipschitz Vs. Element-Wise Lipschitz Activation}
\usepackage{times}
\altauthor{%
 \Name{Amit Daniely} \Email{amit.daniely@mail.huji.ac.il}\\
 \addr School of Computer Science and Engineering, The Hebrew University \\
  and Google Research Tel-Aviv
 \AND
 \Name{Elad Granot} \Email{elad.granot@mail.huji.ac.il}\\
 \addr School of Computer Science and Engineering, The Hebrew University%
}

\begin{document}

\maketitle

\begin{abstract}%
    This study delves into the sample complexity of two-layer neural networks. For a given reference matrix \(W^0 \in \mathbb{R}^{\mathcal{T}\times d}\) (typically representing initial training weights) and an \(O(1)\)-Lipschitz activation function \(\sigma:\mathbb{R}\to\mathbb{R}\), we examine the class
    \[
    \mathcal{H}_{W^0, B, R, r}^{\sigma} = \left\{\textbf{x}\mapsto  \langle\textbf{v},\sigma((W+W^0)\textbf{x})\rangle :
        \|W\|_{\text{Frobenius}} \le R,  \|\textbf{v}\| \le r, \|\textbf{x}\|\le B\right\}.
    \]
    We demonstrate that the sample complexity of \(\mathcal{H}_{W^0, B, R, r}^{\sigma}\) is bounded by
    \[\tilde O \left(\frac{L^2 B^2 r^2\bra{R^2+\|W^0\|^2_{\text{Spectral}}}}{\epsilon^2}\right).\]
    This bound is optimal, barring logarithmic factors, and depends logarithmically on the width \(\mathcal{T}\). This finding improves on \citet{Vardi2022}, who established a similar outcome for \(W^0 = 0\). Our motivation stems from the real-world observation that trained weights often remain close to their initial counterparts, implying that \(\|W\|_{\text{Frobenius}} \ll \|W+W^0\|_{\text{Frobenius}}\). To arrive at our conclusion, we employed and enhanced a recently new norm-based bounds method, the Approximate Description Length (ADL), as proposed by \citet{Daniely2019}.

    Finally, our results underline the crucial role of the element-wise nature of \(\sigma\) for achieving a logarithmic width-dependent bound. We prove that there exists an \(O(1)\)-Lipschitz (non-element-wise) activation function \(\Psi:\mathbb{R}^\mathcal{T}\to\mathbb{R}^\mathcal{T}\) where the sample complexity of \(\mathcal{H}_{W^0, B, R, r}^{\Psi}\) increases linearly with the width.
\end{abstract}


\section{Introduction}

The remarkable capability of Neural Networks (NN) to generalize, even with more parameters than examples, remains a foundational enigma in contemporary NN practice (\cite{Zhang2016}). A recent line of works seek to address this phenomenon through bounds based on the norms of weight vectors, with notable contributions from \cite{Neyshabur2015, Bartlett2017, Golowich2017a, Nagarajan, Daniely2019, Vardi2022}.

First bounds on generalization performance were based on Rademacher Complexity and Covering Numbers, often involving implicit or explicit weight regularization. A breakthrough came with the introduction of the Approximate Description Length (ADL), which proposed a bound that is sub-linear with respect to the number of parameters \cite{Daniely2019}. This research posited a constraint on the deviation of weights from their initialization, suggesting that for constant-depth feed-forward neural networks with a wide set of activation functions, substituting the parameter count with input dimension multiplied by the deviation could yield a more concise asymptotic bound. However, this finding did not accommodate the commonly employed ReLU function, represented by \(\max\brc{\cdot,0}\), thus leaving an unresolved gap.

\cite{Vardi2022} made significant strides by addressing this lacuna for two-layer networks. Their results, obtained via Rademacher Complexity, are tight up to logarithmic factors. Notably, their bound is based on the absolute norm of weights, as opposed to the deviation from their initialization.

The primary contribution of our study is to augment the findings of \cite{Vardi2022}, keeping a similar bound however obtained from the distance from initialization. This challenge, cited as an open question by \cite{Vardi2022}, originates from observations that weight deviations from initialization are often significantly smaller than those from the origin (as evidenced by \cite{Nagarajan, Bartlett2017, Daniely2019}). Our analysis confirms the existence of such a bound for any element-wise \(O(1)\)-Lipschitz activation function.

To substantiate our conclusions, we harness the recent ADL tool introduced by \cite{Daniely2019}. Expanding on this approach, we introduce new methodologies, incorporating a chaining-based strategy tailored for the ReLU activation. We anticipate that these enhanced methods will be instrumental in future research, showcasing the potential power of the ADL framework and catalyzing novel insights.

In the subsequent section, we examine the limits of our assumptions, questioning the extensibility of these bounds to non-point-wise Lipschitz activation functions. Our concluding contribution illustrates the essential role of the element-wise property: we design a non-element-wise Lipschitz activation function and prove lower bounds on the generalization which scale \textbf{linearly} with width.

\section{Preliminaries}

\subsection{Notations}

We denote vectors using bold letters and matrices using upper letters.
We shall add a hat sign $\brb{\hat{\square}}$ or a tilde sign $\brb{\tilde{\square}}$ above letters to mark them as random variables whose expectation equals the letters, e.g., $\E\brb{\hat{\x}} = \x$.

We denote the Frobenius norm of a matrix $W$ by $\|W\|_F^2  = \brp{W,W} =  \sum_{ij}W_{ij}^2$, while the spectral norm is denoted by $\|W\| = \max_{\|\x\| = 1}\|W\x\|$.
We will define $\|\vv\|_\infty$ as the $L^\infty$ norm of the vector $\vv$.
We will use $\log$ with a base of $2$ and $\ln$ with the natural base.

We denote by $\brc{0,1}^k$ a sequence of $k$ bits, and by $\brc{0,1}^* = \bigcup_{k\in\Nat} \brc{0,1}^k$ a sequence of bits of any length.

%Given a vector $\vv\in\R^{\T}$, we denote by $\diag(\vv)$ the $\T\times \T$ matrix whose diagonal is $\vv$ and the other elements are zero.

For any number $a\in\R$, we will denote by $\brdn{a}$ and $\brup{a}$ the floor and ceiling of $a$, respectively.
We will denote by $\brup{a}_+ = \min \brc{n\in \Nat\cup\brc{0} : a \leq n}$.
Note that if $a < 0$ then $\brup{a}_+ = 0$.

We will use the asymptotic notations $O$, $\Theta$, and $\Omega$ to ignore constants and $\tilde{O}$ to ignore logarithmic terms.
We will use $[\lesssim]$ in equations to denote an upper bound up to constant factors.

\subsection{The Two-Layer Model}

Let $\X_B\subseteq\R^{d}$ be a bounded set, s.t. $\forall \x\in\X_B, \brn{\x}\leq B$.
Let $\T\in\Nat$ be the width and $W^0\in\R^{\T\times d}$ be some matrix.
Let $\sigma:\R^{\T}\rightarrow\R^{\T}$ be an $L$-Lipschitz activation function.
For $W\in\R^{\T\times d}$ and $\vv\in\R^\T$ define $h_{W,\vv}:\X_B\to\R$ by $h_{W,\vv}(\x)=\brp{\vv,\sigma(W\x)}$
Finally, given $R,r>0$, consider the following hypothesis class:
\begin{equation}\label{eq:hypothesis_func_def}
    \HH^\sigma_{\T,L,B,R,r} = \biggl\{h_{W,\vv} :\brn{W - W^0}_F \leq R,
	 \brn{\vv } \leq r 
    \biggr\}
\end{equation}
which uses a total of $\T d$ parameters.

Note that while the above definitions do not explicitly mention a bias term in the linear operations, such cases are included in the model, e.g., by forcing the last element of $\x$ to be 1.

%As $W^0$ and $\vv_0$ are known, we consider them as constants for the asymptotic bounds.

\subsection{Approximate Description Length}

Fix a domain $\X$. We say that a random function $\hat{f}:\X\rightarrow\R$ is
an {\em $\epsilon$-estimator} of $f:\X\rightarrow\R^{\T}$ if for
every $x\in \X$, $\E\brb{\hat{f}(x)} = f(x)$ and $\V\bra{\hat{f}(x)} \leq \epsilon^2$.
Fix a hypothesis class $\FF\subset\R^{\X}$. We say that $\FF$ is {\em $\epsilon$-compressible using $n$ bits} if there is a randomized mapping $f\in\FF\mapsto \hat{f}$ such that for any $f\in\FF$, $\hat{f}$ is an $\epsilon$-estimator of $f$ and there is a protocol that given $f$, Alice can randomly encode $\hat f$ using $\le n$ bits in expectation. That is, Alice can send Bob a random string $s$ (that depends on $f$ and Alice's randomness) whose expected length is $\le n$, and then Bob can generate a function $\hat f = f(s)$ such that $\hat f$ is an $\epsilon$-estimator of $f$. In this case, we will say that $\hat f$ is an $\epsilon$-compression of $f$ that uses $n$ bits. In some cases, we will allow the number of bits to depend on $f$ or the parameters defining $f$. We note that Alice can send the empty string, whose length is $0$.

Finally, we will say that $\FF\subset\R^{\X}$ has an {\em approximate description length (ADL)} of $n(m)$ if for any $A\subset \X$ of size $m$, $\FF|_A$ is $1$-compressible using $n(m)$ bits. In \citet{Daniely2019}, it is shown that the ADL bounds the sample complexity:

\begin{theorem}\label{thm:ADL_to_sample_complexity}
Fix a class $\HH$ of functions from $\X$ to $\R$ with ADL $n(m)$ and a label space $\Y$. Fix $L$-Lipschitz and $B$-bounded loss function $\ell:\R\times \Y\to [0,\infty)$.
Then, for any distribution $\DD$ over $\X\times \Y$, with probability at least $1-\delta$ over a choice of a sample set $S\sim \DD^m$,
\[
\sup_{h\in\HH}\ell_{\DD}(h) - \ell_S(h)
\lesssim \frac{(L+B)\sqrt{n(m)} }{\sqrt{m}} \log(m) + B\sqrt{\frac{2\ln\left(2/\delta\right)}{m}}
\]
where $\ell_\DD(h)=\E_{(x,y)\sim\DD}\ell(h(x),y)$ and $\ell_S(h)=\frac{1}{m}\sum_{i=1}^m\ell(h(x_i),y_i)$
\end{theorem}

We will use the following results from \cite{Daniely2019}:

\begin{lemma}\label{lem:averaging}
    Suppose that $\hat f_1,\ldots,\hat f_k$ are i.i.d. $\epsilon$-compressions of $f$ that uses $n$ bits each. Then $\frac{\sum_{i=1}^k\hat{f}_i}{k}$ is an $(\epsilon/\sqrt{k})$-compression of $f$ that uses $O(kn)$ bits.
\end{lemma}

\begin{lemma}\label{lem:adding_compressors}
    Suppose that for any $1\le i\le k$, $\hat f_i$ is an $\epsilon_i$-compression of $f_i$ that uses $n_i$ bits. Assume furthermore that the $\hat f_i$'s are independent. Then $\sum_{i=1}^k\hat{f}_i$ is a $\sqrt{\sum_{i=1}^k\epsilon_i^2}$-compression of $\sum_{i=1}^kf_i$ that uses $O\left(\log(k)\cdot\sum_{i=1}^k n_i\right)$ bits.
\end{lemma}

\begin{lemma}\label{lem:mult_by_constant}
    Suppose that $\HH$ has ADL of $n(m)$ then $C\cdot\HH$ has an ADL of $O((C^2+1)n(m))$
\end{lemma}

\begin{lemma}\label{lem:lin_comp}
Suppose the linear class
\[
\HH = \brc{\lambda_\w:\x\mapsto \brp{\w,\x} : \brn{\w-\w^0} \leq R}
\]
for some initialization $\w^0 \in \R^d$.
Given $\epsilon > 1$, it is possible to $\epsilon$-compress any $\lambda_{\w}$ defined over the set 
$A\subset \X_B$ using $O\left(\frac{Z^2\log\left(dZ\right)}{\epsilon^2}\right)$ bits where $Z = O(B\|\w-\w^0\|)$. 
\end{lemma}

We will also use the following lemmas:

\begin{lemma}\label{lem:finie_bit}
    Fix $A\subset\X_B$ of size $m$ and $\HH, \HH'\subset\R^{\X_B}$ such that 
    \begin{enumerate}
        \item For any $h\in \HH$ there is $h'\in \HH'$ with $\|h-h'\|_\infty \le \delta\le 1$
        \item Assume that $\HH'$ is $1$-compressible using $n$ bits.
    \end{enumerate}
    Then,  $\HH$ is $1$-compressible using $O(n + \delta m\log(m))$ bits.
\end{lemma}
\begin{proof}
    Denote $A = \{\x_1,\ldots,\x_m\}$.
    To compress $h\in \HH$ we will choose $h'\in\HH'$ with $\|h-h'\|_\infty \le \delta$, and generate $1$-compression $\hat h'$ of $h'$ using $n$ bits. Likewise, for any $i\in [m]$ independently choose $i$ w.p.  $|h(\x_i) - h'(\x_i)|$, and let $1_i$ be the indicator of the event that $i$ was chosen. Define $\hat h:A\to [-1,1]$ by $\hat h(\x_i) = \hat h'(\x_i) + \mathrm{sign}(h(\x_i) - h'(\x_i))1_i$. Clearly, for any $i\in[m]$, $\E\hat h(\x_i) = h(\x_i)$. Furthermore
    \[
    \V(\hat h(\x_i)) = \V(\hat h'(\x_i)) + \V(1_i) \le 1 + \delta^2
    \]
    Finally, $\hat h$ can be described using $O(n + \delta m\log(m))$ bits in expectation by concatenating the description of $\hat h'$ and a pair $(i,\mathrm{sign}(h(\x_i) - h'(\x_i)))$ for any $i$ such that $1_i = 1$.
\end{proof}

\begin{corollary}\label{cor:single_param}
    (Single Parameter Compression) Let $\alpha \in \R$. For every $\epsilon\in (0,1)$ there is an $\sqrt\epsilon$-compression for $\alpha$ that uses $O(\log(\brup{\brm{\epsilon\alpha}})$ bits.
\end{corollary}
\begin{proof}
    We will decompose $\alpha = \epsilon\brup{\frac{\alpha}{\epsilon}} + \delta$ where $\delta = \alpha - \epsilon\brup{\frac{\alpha}{\epsilon}} \in (0,1)$.
    We need $O(\log(\brup{\brm{\alpha}/\epsilon})$ to describe $\brup{\alpha/\epsilon}$ (remember that $\epsilon$ is given and known), and from lemma \ref{lem:finie_bit}, an additional $O(1)$ bits for describing $\delta$.
\end{proof}

\subsection{Strong Shattering}

For the lower bound, we will use the notion of \textbf{Strong Shattering}, as defined by \cite{Simon1997}:

\begin{definition}
    A class $\HH\subset \R^\X$ \textbf{strongly-shatters} $x_1,\dots,x_m \in \X$, if there exists $\s\in[0,1]^m$ such that, for every $\bb\in\brc{\pm 1}^m$, there is $h \in \HH$ such that for each $i\in[m]$
    \begin{align*}
        h(x_i) \geq s_i + 1, & \quad \text{ if } b_i = 1 \\
        h(x_i) \leq s_i - 1, & \quad \text{ if } b_i = 0.
    \end{align*}
    We further define $Sdim$ as:
    \[
    Sdim(\HH) = \max\brc{m : \exists x_1,\dots,x_m \in \X, \text{ s.t. } \HH \, \text{ strongly-shatters } x_1,\dots,x_m}.
    \]
\end{definition}

Informally, the $Sdim$ for real-valued functions is like the VC-dimension for $\brc{0,1}$-valued functions.
Previous results (as in \cite{Bartlett1996}) showed the lower bound of the sample complexity scales linearly with $Sdim$.

%%%%%%%%%%%%%%%%%%%%%%%%%%%%%%%%%%%%%%%%%%%%%%%%%

\section {Results and Contributions}

Our first result gives an upper norm-based generalization bound for any element-wise Lipschitz activation function using ADL.

\begin{theorem}\label{thm:main_upper}
	Let $A\subset\X$ of size $m$, and assume $\sigma$ is an element-wise $L$-Lipschitz activation function.
	Then $\HH^\sigma_{\T,L,B,R,r}|_A$ as defined in Eq. \ref{eq:hypothesis_func_def} has an ADL of\footnote{The hidden poly-log factor is $O(\log^3(P))$, where $P$ is the sum of all problem's parameters. See section \ref{sec:upper_proof} for more details.} $\tilde{O}\bra{L^2 B^2 r^2 (R^2 + \|W^0\|^2)}$.
As a result, $\HH^\sigma_{\T,L,B,R,r}$ has a sample complexity of $\tilde{O}\bra{\frac{L^2 B^2 r^2 (R^2 + \|W^0\|^2) }{\epsilon^2}}$.
\end{theorem}

Few remarks about the result:
First, we note that the bound in Theorem \ref{thm:main_upper} is tight, up to a logarithmic factor.
Indeed, if $\sigma$ was the identity function times $L$, then $\HH^{L\cdot Id}_{\T,L,B,R,r}$ would be the hypothesis class of bounded linear functions, which has a known sample complexity of $\tilde\Theta\bra{\frac{L^2 B^2 R^2 r^2}{\epsilon^2}}$ (\cite{Ben-David2014}).

Second, we note that this bound is similar to the upper bound of \cite{Vardi2022}, which showed a bound of $O\bra{\frac{L^2B^2R^2r^2 \log^3(m)}{\epsilon^2}}$, up to logarithmic factors.
The main improvement over their work is that our bound considers the distance of the weights from the initialization $W^0$, which is a more challenging task yet more relevant to the behavior of neural networks in practice.

Third, the proof for the above theorem is based on a new chaining-based argument that extends the ADL approach of \citet{Daniely2019}.
As stated above, \citet{Daniely2019} used this tool to prove a first tight bound up to logarithmic factors for many families of neural networks.
We hope that the techniques in our proof will inspire future works to achieve bounds for deeper networks.

Last, we note that this bound has only logarithmic dependency in the width $\T$.
This raises a natural question: can the element-wise property of $\sigma$ be ignored and still yield the same bounds?
Our second result shows that the answer is negative in general.
Specifically, there is an $O(1)$-Lipschitz function $\hat\sigma:\R^\T\to\R^\T$ for which the class of Eq. \ref{eq:hypothesis_func_def} can be strongly-shattered using $\Theta(\T)$ samples for $\T$ that is up to exponential in $d$.
This brings us to the second result of this paper:

\begin{theorem}\label{thm:main_lower}
    For any dimension $d \geq 20$ and any width $d \leq \T \leq O(e^d) $, there is an $L$-Lipschitz activation function $\bar\sigma$ with $L=32$, and a set of $\Theta(\T)$ samples that strongly shatters the class $\HH^{\bar\sigma}_{\T,L,B=1,R=\sqrt{2d},r=1}$.
\end{theorem}

Note that the parameters $L$, $B$, $R$, and $r$ are independent of $\T$, and yet, by increasing only the width of the hidden layer, the sample-complexity increases similarly. Specifically, when $\T$ is exponential in $d$, the sample complexity of this two-layer network defined by Theorem \ref{thm:main_lower} is $\Omega\bra{e^d}$, whereas if only $\bar\sigma$ would have been element-wise, the same network would be linear in $d$ (i.e., $O(d)$), according to Theorem \ref{thm:main_upper}.

%%%%%%%%%%%%%%%%%%%%%%%%%%%%%%%%%%%%%%%%%%%%

\section{Proof of Theorem \ref{thm:main_upper}}\label{sec:upper_proof}

Let $h_{W,\vv}$ as in equation \ref{eq:hypothesis_func_def}, that is $h_{W,\vv}(\x) = \brp{\vv,\sigma(W\x)}$.
By lemma \ref{lem:mult_by_constant}, we can assume w.l.o.g. that $\sigma$ is $1$-Lipschitz and that $r=1$.
We can decompose $W=\mat{\w_1^T\\ \vdots \\ \w_\T^T}$, where each $\w_i \in \R^d$, and rewrite
\[
h_{W,\vv} = \sum_{i=1}^\T v_i \sigma(\w_i^T \x).
\]
As $\sigma$ is an element-wise function, one can create statistically independent estimators for each expression $v_i\sigma(\w_i^T\x)$ in the sum.
Moreover, if each estimator is $\epsilon_i$-compressible using $n_i$ bits, then using lemma \ref{lem:adding_compressors} we get an $\sqrt{\sum_{i=1}^\T \epsilon_i^2}$-compression for $h_{W,\vv}$ using $
log\bra{\T} \sum_{i=1}^\T n_i$ bits.
The following proof shows how to construct such compressors with $\epsilon_i^2$ that scales as $\frac{v_i^2}{r^2}$ and $n_i$ that scales as $\frac{\brn{\w_i-\w_i^0}^2}{R^2}$, hence ommitting the need for $\T$ up to logarithmic factor.

Fix a set $A\subset \X_B$ of size $m$ and a vector $\w^0 \in \R^d$.
For $\w\in\R^d$ and $v\in\R$ we define $h_{\w,v}:A\to\R$ by $h_{\w,v}(\x) = v(\sigma(\brp{\w,\x}) - \sigma(\brp{\w^0,\x}))$ and consider the class of single-neuron networks:
\begin{align*}
\HH = \brc{ h_{\w,v} :
	\w \in \R^d,v\in \R}.
\end{align*}

We will show how to $|v|$-compress any $h_{\w,v}$ using $\tilde O\bra{B^2\brn{\w-\w^0}^2}$ bits.

From claim \ref{lem:lin_comp} we can get an $\epsilon$-compression for $\lambda_\w$. Define this compression as $\hat\w(\epsilon)$.
We seem to be on a good track to compress $\sigma(\brp{\w,\x})$. However, this is misleading. Indeed, if we use $\hat\w(\epsilon)$ to create the random variable $\sigma(\brp{\hat\w(\epsilon),\x})$, we will not get an $\epsilon$-estimation, as $\E\brb{\sigma(\brp{\hat\w(\epsilon),\x})} \neq \sigma(\brp{\w,\x})$ for many choices of $\sigma$.
We will, therefore, need a different approach.

Let us move to a non-efficient yet straight-forward approach:
Recall that $A\subset\R^{\X_B}$ is fixed, with $\brm{A}=m$. Hence, the function $\sigma(\brp{\w,\x})$ can get up to $m$ different results. 
Using corollary \ref{cor:single_param}, we can $\epsilon$-compress each such value using at most $Z=O(\log(BR/\epsilon))$ bits, and a total of $mZ$ bits to $\epsilon$-compress the entire function.
However, this does not seem like an optimal compression, as the number of bits is linear in $m$, which we want to avoid.
Yet, we will still use this approach in our construction:
Let $k\in\Nat$ that will be defined later, and set $\epsilon_k=2^{-k/2}B\brn{\w-\w^0}$.
Based on the above, we'll construct an $\epsilon_k$-compression $\hat{g}$ of the function $g:\x\mapsto \sigma(\brp{\w,\x})-\E\brb{\sigma\bra{\brp{\hat\w(\epsilon_k),\x}}}$ using $mZ_k$ bits, were $Z_k=O(\log(2^kB\brn{\w-\w^0}))=O(k\log(B\brn{\w-\w^0}))$.

With this compression at hand, we proceed with the following scheme:
\begin{itemize}
    \item Let $\hat{v}$ a $\brm{v}$-compression for $v$. From lemma \ref{cor:single_param} exists such a compression that uses $O(\log(r))$ bits.
    \item Given $k\in\Nat$, choose $i\in \{1,\ldots,k+1\}$ such that the probability to choose  $1\le i\le k$ is $2^{-i}$, and the probability to choose $k+1$ is $2^{-k}$.
    \item If $i = 1$, create the random variable $\hat\w\bra{\epsilon_1}$ with $\epsilon_1 = 2^{-1/2}B\brn{\w-\w^0}$.
    Set $\hat{f}$ as the function $\hat{f}(\x) = 2(\sigma(\brp{\hat \w(\epsilon_{1}),\x})-\sigma(\brp{\w^0,\x}))$.
    \item If $2 \leq i \leq k$, create two independent random variables, $\hat\w\bra{\epsilon_{i-1}}$ and $\hat\w\bra{\epsilon_{i}}$ where $\epsilon_i = 2^{-i/2}B\brn{\w-\w^0}$.
    Set $\hat{f}$ as the function $\hat{f}(\x) = 2^i(\sigma(\brp{\hat \w(\epsilon_{i}),\x}) - \sigma(\brp{\hat \w(\epsilon_{i-1}),\x}))$.
    \item If $i = k + 1$ then generate $\hat\w(\epsilon_{k})$ and define $\hat f(\x) = 2^k\hat{g}(\x)$.
    \item Output $\hat h = \hat{v}\hat{f}$.
\end{itemize}

The idea behind the structure above is to create a chain of events with increasing accuracy and cost (in number of bits) but with a decreasing probability of occurring.
The following claim, together with lemma \ref{lem:averaging} shows that it is possible to $\brm{v}$-compress $h_{\w,v}$ using
$\tilde O(B^2\|\w-\w^0\|^2)$ bits.

\begin{claim}
    For $k=\log_2(m)$ we have that $\hat h$ is a $O(\brm{v}B\|\w-\w^0\|\sqrt{\log(m)})$-compression for $h_{\w,v}$ that uses $O(\log(dZ)\log(m))$ bits, for $Z$ as defined in lemma \ref{lem:lin_comp}. 
\end{claim}
\begin{proof}
Fix $\x\in A$. We need to show that $\E\brb{\hat h(\x)} = h_{\w,v}(\x) = v(\sigma(\brp{\w,\x}) - v\sigma(\brp{\w^0,\x})$, $\mathrm{Var}(\hat{f}(\x))\le O(v^2B^2\brn{\w-\w^0}\log(m))$, and that the number of bits that are used is $O(\log(dZ)\log(m))$.
Indeed, since $\hat{v}$ is independent from the rest of the random variables and $\E\brb{\hat v} = v$, we get:
\begin{eqnarray*}
    \frac{1}{v}\E\brb{\hat h(\x)}
    &=& \frac{1}{v} \E\brb{\hat{v}} \E\brb{\hat{f}}
    \\
    &=& \E\brb{\sigma(\brp{\hat\w(\epsilon_1),\x}) - \sigma(\brp{\w^0,\x})}
    \\
    && + \sum_{i=2}^k 2^{-i}\E\brb{2^i(\sigma(\brp{\hat \w(\epsilon_{i}),\x})-\sigma(\brp{\hat \w(\epsilon_{i-1}),\x}))}
    \\
    && + 2^{-k}\E\brb{2^k\hat{g}(\x)}
    \\
    &=& \sigma(\brp{ \w,\x}) - \sigma(\brp{\w^0,\x}).
\end{eqnarray*}

Likewise,
\begin{eqnarray*}
    \V(\hat h(\x))
    &\stackrel{\mathrm{Var}(X)\le \E X^2}{\le}& 
    \E\brb{\hat v^2 (\hat f(\x))^2}
    \\
    &\stackrel{\text{$\hat v$ independent of $\hat f$}}{=}&
    \bra{\V(\hat v)+\E\brb{\hat v}^2}\E\brb{\bra{\hat f(\x)}^2}
    \\
    &=& 2v^2\E\brb{\bra{\hat f(\x)}^2}
\end{eqnarray*}
and as $\E\brb{\hat{g}(\x)}=g(\x)$ and $\V(\hat g(\x)) \leq \epsilon_k$, we get:
\begin{eqnarray*}
\E\brb{\bra{\hat f(\x)}^2}
&=& 2^{-1}\E\brb{4(\sigma(\brp{\hat \w(\epsilon_1),\x}) - \sigma(\brp{\w^0,\x})^2}
\\
&&+ \sum_{i=2}^k2^{-i}\E\brb{2^{2i}\left(\sigma(\brp{\hat \w(\epsilon_{i}),\x})-\sigma(\brp{\hat \w(\epsilon_{i-1}),\x})\right)^2}  
\\
&&+ 2^{-k}\bra{\bra{\E\brb{2^{k}\hat g(\x)}}^2 + \V\bra{2^{k}\hat g(\x)}}
\\
&\stackrel{\sigma\text{ is $1$-Lipschitz}}{\le}&
2\E\brb{(\brp{\hat \w(\epsilon_1),\x} -  \brp{\w^0,\x})^2}
\\
&& + \sum_{i=2}^k2^i\E\brb{\left(\brp{\hat \w(\epsilon_{i}),\x}-\brp{\hat \w(\epsilon_{i-1}),\x}\right)^2}
\\
&&+ 2^{k}\bra{\bra{\E\brb{\brm{\brp{\w,\x} - \brp{\hat \w(\epsilon_{k}),\x}}}}^2 + 2^{-k}B^2\brn{\w-\w^0}^2}
\\
&\stackrel{\text{Jensen's Inequality}}{\le}&
2\E\brb{(\brp{\hat \w(\epsilon_1),\x} - \brp{\w,\x} + \brp{\w,\x} - \brp{\w^0,\x})^2}
\\
&& + \sum_{i=2}^k2^i\E\brb{\left(\brp{\hat \w(\epsilon_{i}),\x} - \brp{\w,\x} + \brp{\w,\x} - \brp{\hat \w(\epsilon_{i-1}),\x}\right)^2}
\\
&&+ 2^{k}\E\brb{\bra{\brp{\w,\x} - \brp{\hat \w(\epsilon_{k}),\x}}^2} + B^2\brn{\w-\w^0}^2
\\
&\stackrel{(*)}{\le}&
2\brp{\w-\w^0,\x}^2 + B^2\brn{\w-\w^0}^2
\\
&& +3\sum_{i=1}^k 2^i\V\bra{\brp{\hat \w(\epsilon_{i}),\x}}
\\
&\stackrel{\text{$\epsilon_i$-compressors}}{\le}&
3 B^2\brn{\w-\w^0}^2
\\
&& +3\sum_{i=2}^k2^i2^{-i}B^2\brn{\w-\w^0}^2 
\\
&\le & 6kB^2\brn{\w-\w^0}^2.
\end{eqnarray*}
Note that the step marked with $(*)$ follows from the independence between $\brc{\hat\w(\epsilon_i)}$, as:
\begin{eqnarray*}
    \E\brb{\left(\brp{\hat \w(\epsilon_{i})-\w,\x} - \brp{\hat \w(\epsilon_{i-1})-\w,\x}\right)^2}
    &=& \E\brb{\brp{\hat \w(\epsilon_{i})-\w,\x}^2}
    + \E\brb{\brp{\hat \w(\epsilon_{i-1})-\w,\x}^2}
    \\
    && -2\E\brb{\brp{\hat \w(\epsilon_{i})-\w,\x}}\E\brb{\brp{\hat \w(\epsilon_{i-1})-\w,\x}}
    \\
    &=&
    \V\bra{\brp{\hat \w(\epsilon_{i}),\x}} + \V\bra{\brp{\hat \w(\epsilon_{i-1}),\x}}
\end{eqnarray*}
where the last equality follows since $\E\brb{\brp{\hat \w(\epsilon_{i}),\x}}=\brp{\w,\x}$.

Finally, from lemma \ref{lem:lin_comp}, the expected number of bits that are required, up to a constant factor, is
\[
\sum_{i=1}^k 2^-i\frac{Z^2\log(dZ)}{\epsilon_i^2} + 2^{-k}mZ_k=
k\log(dZ) +2^{-k}mZ_k.
\]

When setting $k=\log_2(m)$ we get an $O(\brm{v}B\|\w-\w^0\|\sqrt{\log(m)})$-compression for $h_{\w,v}$ that uses $O(\log(dZ)\log(m))$ bits.
\end{proof}

\begin{corollary}
    Using the above claim and lemma \ref{lem:averaging}, we can construct a $\brm{v}$-compression for $h_{\w,v}$ that uses
    $O(B^2\brn{\w-\w^0}^2log(dZ)log^2(m))$ bits.
    Then, using lemma \ref{lem:adding_compressors} we can compose a $1$-compression for $h_{W,\vv} - \vv\sigma(W^0,\x)$ that uses $O(B^2R^2r^2log(dZ)log^2(m))$ bits. Finally, we can use lemma \ref{lem:lin_comp} to create a $1$-compression for $\vv\sigma(W^0,\x)$ with an addition of $O(B^2\brn{W^0}^2r^2log(dZ))$ bits.
\end{corollary}

%%%%%%%%%%%%%%%%%%%%%%%%%%%%%%%%%%%%%%%%%%%%%%%%%%%%%%%%%%%%%

\section{Proof of Theorem \ref{thm:main_lower}}

Theorem \ref{thm:main_upper} shows a generalization bound of the class $\HH^{\sigma}_{\T, L, B, R, r}$ as in Eq. \ref{eq:hypothesis_func_def} with a neglectible logarithmic dependency in the width, $\T$.
The above is true, however, when $\sigma$ is $L$-Lipschitz element-wise function.
What if $\sigma$ was not element-wise?

In this section, we'll proof Theorem \ref{thm:main_lower} that shows that removing the element-wise property results in a bound that is \textbf{linearly} dependent on the width.
We will show that there is a $\Theta(1)$-Lipschitz function, $\bar{\sigma}$, such the class $\HH^{\bar\sigma}_{\T, L, B, R, r}$ can be strongly shattered using $\Theta(\T)$ samples, when $B, R, r$ depends only at the input dimension, $d$.
We then conclude that the sample complexity of Theorem \ref{thm:main_upper} cannot be achievable in the non-element-wise case.

The proof is constructive and shows that by picking $m=\Theta(\T)$ samples $\x^1,\dots,\x^m \in \X_B$ and $2^m$ matrices $W^1,\dots,W^{2^m} \in \R^{\T\times d}$ at random, the 
the set of points $P:= \brc{W^k \x^i : i\in[m], k\in [2^m]}$ are far enough from each other with a positive probability (The details are presented in Lemma \ref{lem:far_rv}).

Hence, we can construct an activation function, described in Lemma \ref{lem:lipschitz}, that can move every desire point in $P$ to a vector of our choice, while maintaining the Lipschitzness property.

Finally, we conclude that there is a set of samples $\x^1,\dots,\x^m \in \X_B$ that are able to strongly-shatter $\HH^{\bar\sigma}_{\T, L, \Theta(1), \Theta(d), \Theta(1)}$.
Note that the width $\T$ can be exponentially big with respect to $d$, and the number of shattered samples, $m$, grows linearly with it.

Denote by $\vol_k(A)$ the $k$-dimensional volume of a set $A\subset \R^d$ normalized such that the volume $\vol_{\T-1}\left(\Scirc^{\T-1}\right)=1$. Denote also $B^d(\x,R) = \{\x'\in\R^d : \|\x-\x'\|\le R\}$.
We will use the following fact
\begin{lemma}\label{lem:sphere_cover}
    For any $\x\in\Scirc^{\T-1}$ and sufficiently large $\T$ we have 
    \[
    \vol_{\T-1}\left(\Scirc^{\T-1}\cap B^\T(\x,1/2)\right) \le e^{-\frac{\T}{3}} < \T^{-2} 2^{-\T/4}
    \]
\end{lemma}
\begin{proof}
Denote $\epsilon = \frac{1}{2}$. We have
    \[
    \|\x-\w\|^2 <\epsilon^2\Leftrightarrow  2-2\brp{\w,\x}<\epsilon^2 \Leftrightarrow \brp{\w,\x} > 1-\epsilon^2/2
    \]
    Hence,    
    \[
    \Scirc^{\T-1}\cap B^\T(\x,\epsilon) = \{\w\in \Scirc^{\T-1} : \brp{\w,\x}>1-\epsilon^2/2\}
    \]
    Let $\w\in \Scirc^{\T-1}$ be a uniform vector. For any $a>0$ we have  $\Pr(\brp{\w,\x}>a)\le 2e^{-\frac{\T a^2}{2}}$ (e.g. chapter 14 in \cite{matousek2013lectures}). Hence,
    \begin{align*}
    \vol_{\T-1}\left(\Scirc^{\T-1}\cap B^\T(\x,\epsilon)\right)
    &= \Pr(\brp{\w,\x}>1-\epsilon^2/2) \\
    &\le 2e^{-\frac{\T (1-\epsilon^2/2)^2}{2}}
    = 2e^{-\frac{\T 49}{128}} \le (e/2)^{-\T / 4}2^{-\T / 4}
    \end{align*}
    this concludes the proof as $(e/2)^{-\T / 4}\le \T^{-2}$ for sufficiently large $\T$.
\end{proof}

\begin{lemma}\label{lem:far_rv}
	For $e^{d/3} \ge \T\ge d \geq 20$,
	there exists a set of vectors $\x^1, \dots, \x^{m} \in \Scirc^{d-1}$
	and a set of matrices $A^1, \dots A^{2^m} \in \R^{\T \times d}$
	that have the following properties:
	\begin{enumerate}
		\item $m =  \left\lfloor\T/4\right\rfloor$
		\item $A^s$ in an isometry (and hence $\brn{A^s}^2_F = d$), for each $s\in[2^m]$
		\item $\brn{A^s \x^i - A^t \x^j} \geq \frac{1}{2}$, for each $i,j\in[m]$ and $s,t\in[2^m]$ such that $(s,i)\neq(t,j)$
	\end{enumerate}
\end{lemma}

\begin{proof}
Choose $m$ vectors $\x_1,\ldots,\x_m\in\Scirc^{d-1}$ such that $\|\x_i-\x_j\|\ge \frac{1}{2}$ if $i\ne j$. By lemma \ref{lem:sphere_cover} this is possible as long as $\T/4 \le e^{d/3}$. Let $A^1,\ldots,A^k$ be the maximal set of matrices that satisfy items 2. and 3. We need to show that $k\ge 2^m$.

Let $A \in \R^{\T\times d}$ be a random matrix chosen 
uniformly from the set of matrices with unit norm columns 
that are orthogonal to one another. We have $A$ is an isometry with  $\|A\|_F = \sqrt{d}$. Furthermore, adding $A$ to $A^1,\ldots,A^k$ will violate item 2. or 3. only if $\brn{A \x^i - A^t \x^j} < \frac{1}{2}$ for some $i,j\in [m]$ and $t\in [k]$. Since $A \x^i$ is a uniform vector in $\Scirc^{\T - 1}$, the probability of violation is bounded by $k m^2 \T^{-2} 2^{-\T/4} \le k 2^{-\T/4}$. On the other hand, by the maximality of $k$, this probability is $1$. This implies that $k\ge 2^{\T/4}\ge 2^m$.
\end{proof}

\begin{lemma}\label{lem:lipschitz}
	Let $x_1, \dots, x_m$ be a finite set of different points in some metric space $(\X, d)$, such that for each $i\neq j \in [m]$, $d(x_i, x_j) \geq \alpha$.
	Let further be $p_1, \dots, p_m \in \R$ any set of points.
	Then there exists an $L$-Lipschitz function, $f: \X \rightarrow \R$, where
	\[
	L = \frac{2}{\alpha} \min_{C \in\R}\max_{i\in[m]} (\brm{p_i - C}).
	\]
	such that for each $i\in[m]$, $f(x_i) = p_i$.
\end{lemma}

\begin{proof}
Choose $C$ such that $L = \frac{2}{\alpha} \max_{i\in[m]} (\brm{p_i - C})$ and define
	\[
	f(x) = \max_{i\in[m]} \brc{p_i - L d(x, x_i)} 
	\]
 $f$ is $L$-Lipschitz as a maximum of $L$-Lipschitz functions. Fix $x_j$. It is enough to show that $f(x_j) = p_j$. First, $f(x_j) \ge p_j - L d(x_j, x_j) = p_j $. Thus, it remain to show that $f(x_j) \le p_j$. Fix some $i\in [m]\setminus\{j\}$ it is enough to show that $p_i - L d(x_j, x_i) \le p_j$. Indeed,
\begin{eqnarray*}
     p_i - L d(x_j, x_i) &\stackrel{d(x_j,x_i)\ge \alpha}{\le}& p_i - L\alpha
     \\
     &\stackrel{\text{definition of }L}{=}& p_i - 2\max_{i\in[m]} (\brm{p_i - C})
     \\
     &\le& p_i - (|p_i - C| + |p_j - C|)
     \\
     &=& C + (p_i - C) - (|p_i - C| + |p_j - C|)
     \\
     &\le & C - |p_j - C|
     \\
     &\le & C - (C - p_j) = p_j
\end{eqnarray*}
 
\end{proof}

We are now ready to prove the main theorem.

\begin{proof}(of Theorem \ref{thm:main_lower})
    Based on the previous lemmas, we'll strongly shatter a set of $m=\frac{\T}{10}$ samples.
    
    Order the elements of the set $2^{[m]}$ as $S_1, \dots S_{2^m}$ in some arbitrary order,
    and define the function $f:[m]\times[2^m] \rightarrow \brc{\pm 1}$ as:
    \[
    f(k, i) =
    \begin{cases}
    	1, & i \in S_k \\
    	-1, & i \notin S_k
    \end{cases}
    \qquad \forall i\in[m], k\in[2^m].
    \]
    
    Let $\x^1, \dots, \x^m \in \Scirc^{d-1}$
    and $A^1, \dots A^{2^m} \in \R^{\T \times d}$
    be the sets defined in lemma \ref{lem:far_rv},
    and note from the lemma that the set
    $Q = \brc{A^s \x^i : i\in[m], s\in[2^m]}$
    contains $m2^m$ different elements such that for each pair
    $A^s \x^i \neq A^t \x^j$ we have
    \begin{align*}
    	\brn{A^s \x^i - A^t \x^j} \geq \frac{1}{2}.
    \end{align*}
    We can now apply Lemma \ref{lem:lipschitz} with the Euclidean metric space, to get a $4$-Lipschitz function, $\hat{f}:\R^d\rightarrow\R$, such that for all $i\in[m], k\in[2^m]$,
    \[
    \hat{f}(A^k \x^i) = f(k, i).
    \]
    The activation function will therefore be $\bar\sigma(\vv)=\hat{f}(\vv)\e_1$ (or alternatively, we can distribute $\hat{f}$ evenly over the all the $\T$ hidden neurons).
    
    Finally, as each $\brn{A^s}^2_F = d$, we can create the hypothesis class (using the definition of Eq. \ref{eq:hypothesis_func_def}):
    \[
    \HH^{\bar\sigma}_{\T, 4, 1, \sqrt{d}, 1} \supset \brc{\x \mapsto \e_1^T \bar\sigma(A x) : A\in\R^{\T\times d}, \brn{A}_F \leq d}
    \]
    and note that it defines a neural network that can 1-shatter the $m$ points, and $m = \Theta(\T)$.
\end{proof}

%%%%%%%%%%%%%%%%%%%%%%%%%%%%%%%%%

\section{Discussion and Open Questions}

This work aims to understand the sample complexity of depth-two neural networks and the effect of element-wise activation functions (i.e., functions that work on each neuron independently) on the sample complexity of neural networks.
Using the ADL approach, we have shown that this property is sufficient and necessary for two-layer networks to achieve logarithmic width-dependency bounds. By necessary, we mean that the set of general non-pointwise Lipschitz contains activations under which the sample complexity is larger than any element-wise Lipschitz activation functions.
One can view a non-element-wise Lipschitz function as a set of neurons that can share knowledge. Our work shows that this ability amplifies the sample complexity of the network.
%there are non-element-wise functions whose bounds are width-dependent.

We note that the upper bound presented this work is tight w.r.t. to all parameters (i.e., $\T$, $L$, $B$, $d$, $R$, $r$, $\brn{W^0}$ and $\epsilon$). To the best of our knowledge, such a tightness is not implied by previous results.
The optimally of the dependence on $L, B, R, r$ and $\epsilon$ is true already for non relative bound, as discussed in \cite{Vardi2022}.
As for the spectral norm of $W^0$, note that even if $R=0$, $\mathcal{H}^{\sigma}_{\mathcal{T},L,B,0,r}$
contains linear classifiers of norm $r$ over examples of norm $B\|W_0\|$, which yields a sample complexity at least $(LBr\|W_0\|)^2$.
Finally, the tightness of $\T$ is shown by the upper and lower bounds of this work.

Additional to the above, in this work we have developed a new technique that extends ADL and creates a chain of events with increasing accuracy but with a decreasing probability of occurring. This provides better control over both competing values: the variance and the number of bits. We hope that this idea will spark following works.

We are still left with two open questions, one for sufficiency and one for necessity.
Regarding sufficiency, a natural question is whether the results in the paper can be extended to deeper networks.
\citet{Daniely2019} gave a hint for this question, showing a sample complexity for deep neural networks that require only the sum of the widths (which is sublinear in the number of parameters).
Yet, their result does not hold for any element-wise Lipschitz activation function.
We believe achieving similar bounds for any element-wise Lipschitz activation function is possible.

As for the necessity, we note that our lower bound is not valid for {\em any} non-element-wise Lipschitz activation function.
Indeed, if we take some permutation of an element-wise activation function, we do not expect to get width-dependent bounds, although we lost the element-wise property.
Instead, we want to ask whether there exists a (non-element-wise) Lipschitz activation function that guarantees a linear lower bound in the number of parameters, hence matching the upper bound obtained via the "parameters counting" approach.
In our result, the lower bound is still sublinear in the number of parameters.
%In our results the width can have finitely many neurons, but our lower bound is still sublinear in the number of parameters.

%Our last note is dedicated to discussing the potential power of the Approximate Description Length (ADL) approach.
%We found this approach as a new angle for analyzing the sample complexity using a statistical point of view.
%This work increases the set of tools for ADL, and we hope that they will keep bearing fruits in later works of norm-based bounds.

%\section*{References}

% \bibliographystyle{plainnat}
\bibliography{elementwise}

%%%%%%%%%%%%%%%%%%%%%%%%%%%%%%%%%%%%%%%%

\end{document}